\documentclass{new_tlp}
\usepackage{mathptmx}

\usepackage{amssymb}
\usepackage{graphicx}
\usepackage{multirow}

\usepackage[disable]{todonotes} 
\usepackage{url}
\usepackage{cprotect}
\usepackage{pgfplots}
\pgfplotsset{compat=1.16}

\usepackage{fancyvrb}
\usepackage{fvextra}

\usepackage{apacite}

\newcommand{\nop}[1]{}

\newcommand{\reachablein}[1]{\calR_{#1}}
\newcommand{\precond}[1]{\mathit{pre}({#1})}
\newcommand{\addeffects}[1]{\mathit{add}({#1})}
\newcommand{\deleffects}[1]{\mathit{del}({#1})}
\newcommand{\applyactions}[2]{#1[#2]}

\newenvironment{changemargin}[2]{%
\list{}{\rightmargin#2\leftmargin#1
\parsep=0pt\topsep=0pt\partopsep=0pt}
\item[]}
{\endlist}

\newtheorem{theorem}{Theorem}

\newtheorem{definition}[theorem]{Definition}
\newtheorem{example}[theorem]{Example}

\let\phi\varphi
\let\epsilon\varepsilon

\renewcommand{\models}{\vDash}

\newcommand{\calA}{\mathcal{A}}

\newcommand{\calF}{\mathcal{F}}

\newcommand{\calI}{\mathcal{I}}

\newcommand{\calR}{\mathcal{R}}

\newcommand{\NP}{\ensuremath{\textsc{NP}}}

\newcommand{\SIGMA}[2]{\ensuremath{\Sigma_{\mathit{#2}}^{\mathit{#1}}}}

\newcommand{\PSPACE}{\ensuremath{\textsc{PSpace}}}

\newcommand{\answersets}[1]{\mathit{AS}(#1)}

\newcommand{\eK}{\mathbf{K}\,}
\newcommand{\eM}{\mathbf{M}\,}

\newcommand{\citep}[1]{\citeA{#1}}

\title[Action Reversibility in STRIPS Using ASP and ELP]
{Determining ActionReversibility in STRIPS Using Answer Set and Epistemic
Logic Programming}

\author[W.\ Faber, M.\ Morak, and L.\ Chrpa]
       {WOLFGANG FABER$^a$, MICHAEL MORAK$^a$, and LUK\'{A}\v{S} CHRPA$^b$\\
       $^a$University of Klagenfurt, Austria\\
       $^b$Czech Technical University in Prague\\
       \email{wolfgang.faber@aau.at, michael.morak@aau.at, chrpaluk@fel.cvut.cz}}

\jdate{March 2003}
\pubyear{2003}
\pagerange{\pageref{firstpage}--\pageref{lastpage}}
\doi{S1471068401001193}

\begin{document}

\label{firstpage}

\maketitle

\begin{abstract}
	In the context of planning and reasoning about actions and change, we call an
action reversible when its effects can be reverted by applying other actions,
returning to the original state. Renewed interest in this area has led to
several results in the context of the PDDL language, widely used for describing
planning tasks.

In this paper, we propose several solutions to the computational problem of
deciding the reversibility of an action. In particular, we leverage an existing
translation from PDDL to Answer Set Programming (ASP), and then use several
different encodings to tackle the problem of action reversibility for the STRIPS
fragment of PDDL. For these, we use ASP, as well as Epistemic Logic Programming
(ELP), an extension of ASP with epistemic operators, and compare and contrast
their strengths and weaknesses.

Under consideration for acceptance in TPLP.

	\footnote{This work is based on, and significantly extends, two workshop papers previously published by the authors at ASPOCP'20 and EELP'20 \cite{aspocp:ChrpaFFM20,eelp:FaberM20}.}
\end{abstract}


\section{Introduction}\label{sec:introduction}

Traditionally, the field of Automated Planning deals with
the problem of generating a sequence of actions---a plan---that transforms an
initial state of the environment to some goal state, see for instance \cite{0014222,GNT2016}. Actions, in plain words,
stand for modifiers of the environment. One interesting question is whether the
effects of an action are reversible (by other actions), or in other words,
whether the action effects can be undone. Notions of reversibility have
previously been investigated, most notably by \cite{japll:EiterEF08} and by \cite{icaps:DaumT0HW16}.

Studying action reversibility is important for several reasons. Intuitively,
actions whose effects cannot be reversed might lead to dead-end states from
which the goal state is no longer reachable. Early detection of a dead-end state
is beneficial in a plan generation process, as shown by \cite{LipovetzkyMG16}. Reasoning in
more complex structures such as Agent Planning Programs~\citep{GiacomoGPSS16},
which represent networks of planning tasks where a goal state of one task is an
initial state of another task is even more prone to dead-ends, as shown by~\cite{ChrpaLS17}.
Concerning non-deterministic planning, for instance Fully Observable
Non-Deterministic (FOND) Planning, where actions have non-deterministic effects,
determining reversibility or irreversibility of each set of effects of the
action can contribute to early dead-end detection, or to generalise recovery
from undesirable action effects, which is important for efficient computation of
strong (cyclic) plans, cf.~\citep{CamachoMM16}. Concerning online planning, we can
observe that applying reversible actions is safe and hence we might not need to
explicitly provide the information about safe states of the
environment~\citep{CsernaDRR18}. Another, although not very obvious, benefit of
action reversibility is in plan optimization. If the effects of an action are
later reversed by a sequence of other actions in a plan, these actions might be
removed from the plan, potentially shortening it significantly. It has been
shown by \cite{ChrpaMO12} that under given circumstances, pairs of inverse actions, which are a
special case of action reversibility, can be removed from
plans.

\cite{MorakCFF20} introduced a general framework for action reversibility
that offers a broad definition of the term, and generalises many of the already
proposed notions of reversibility, like ``undoability'' proposed by
\cite{icaps:DaumT0HW16}, or the concept of
``reverse plans'' as introduced by \cite{japll:EiterEF08}. The concept of reversibility of \cite{MorakCFF20} 
directly incorporates the set of states in which a given action should be
reversible. We call these notions $S$-reversibility and $\phi$-reversibility,
where the set $S$ contains states, and the formula $\phi$ describes a set of
states in terms of propositional logic. These notions are then further
refined to universal reversibility (referring to the set of all states) and to
reversibility in some planning task $\Pi$ (referring to the set of all reachable
states w.r.t.\ the initial state specified in $\Pi$). These last two versions
match the ones proposed by \cite{icaps:DaumT0HW16}. Furthermore, our notions can be further
restricted to require that some action is reversible by a single ``reverse
plan'' that is not dependent of the state for which the action is reversible. For
single actions, this matches the concept of the same name proposed by
\cite{japll:EiterEF08}.

The complexity analysis of \cite{MorakCFF20} indicates that some of these problems can be addressed by means of Answer Set Programming (ASP), but also by means of Epistemic Logic Programs (ELPs). In this paper, we
leverage the translations implemented in plasp~\citep{DimopoulosGLRS19}, and
produce ASP and ELP encodings to effectively solve some of the reversibility problems on PDDL domains,
restricted, for now, to the STRIPS fragment~\citep{Fikes71}. The encodings differ quite a bit concerning their generality and extensibility, and we discuss their advantages and disadvantages.
We also present preliminary experiments that compare the various encodings, highlighting a trade-off between extensibility and efficiency.

\section{Background}\label{sec:preliminaries}


\paragraph{STRIPS Planning.}
Let $\calF$ be a set of \emph{facts}, that is, propositional variables describing
the environment, which can either be true or false. Then, a subset $s \subseteq
\calF$ is called a \emph{state}, which intuitively represents a set of facts
considered to be true. An action is a tuple $a = \langle \precond{a},
\addeffects{a}, \deleffects{a} \rangle$, where $\precond{a} \subseteq \calF$ is
the set of \emph{preconditions} of $a$, and $\addeffects{a} \subseteq \calF$ and
$\deleffects{a} \subseteq \calF$ are the add and delete effects of $a$,
respectively. W.l.o.g., we assume actions to be well-formed, that is,
$\addeffects{a} \cap \deleffects{a} = \emptyset$ and $\precond{a} \cap
\addeffects{a} = \emptyset$. An action $a$ is \emph{applicable} in a state $s$
iff $\precond{a} \subseteq s$.  The result of applying an action $a$ in a state
$s$, given that $a$ is applicable in $s$, is the state $\applyactions{a}{s} = (s
\setminus \deleffects{a}) \cup \addeffects{a}$. A sequence of actions $\pi =
\langle a_1, \ldots, a_n \rangle$ is applicable in a state $s_0$ iff there is a
sequence of states $\langle s_1, \ldots, s_n \rangle$ such that, for $0 < i \leq
n$, it holds that $a_i$ is applicable in $s_{i-1}$ and
$\applyactions{a_i}{s_{i-1}} = s_i$. Applying the action sequence $\pi$ on $s_0$
is denoted $\applyactions{\pi}{s_0}$, with $\applyactions{\pi}{s_0} = s_n$. The
\emph{length} of action sequence $\pi$ is denoted $|\pi|$.

A \emph{STRIPS planning task} 
$\Pi = \langle \calF, \calA, s_0, G \rangle$
is a four-element tuple consisting of a set of \emph{facts} $\calF = \{f_1, \ldots,
f_n\}$, a set of \emph{actions} $\calA = \{a_1, \ldots, a_m\}$, an
\emph{initial state} $s_0 \subseteq \calF$, and a \emph{goal} $G \subseteq \calF$. A state $s \subseteq \calF$ is a
\emph{goal state (for $\Pi$)} iff $G \subseteq s$. An action sequence $\pi$ is
called a \emph{plan} iff $\applyactions{\pi}{s_0} \supseteq G$.  We further
define several relevant notions w.r.t.\ a planning task $\Pi$. A state $s$ is
\emph{reachable from state $s'$} iff there exists an applicable action sequence
$\pi$ such that $\applyactions{\pi}{s'} = s$. A state $s \in 2^\calF$ is simply
called \emph{reachable} iff it is reachable from the initial state $s_0$. The
set of all reachable states in $\Pi$ is denoted by $\reachablein{\Pi}$. An
action $a$ is \emph{reachable} iff there is some state $s \in \reachablein{\Pi}$
such that $a$ is applicable in $s$.

Deciding whether a STRIPS planning task has a plan is known to be
\PSPACE-complete in general and it is \NP-complete if the length of the plan is polynomially bounded~\citep{ai:Bylander94}.

\paragraph{Epistemic Logic Programs (ELPs) and Answer Set Programming (ASP).} We
assume the reader is familiar with ELPs and will only give a very brief overview
of the core language. For more information, we refer to the original paper
proposing ELPs~\citep{aaai:Gelfond91}, therein named \emph{Epistemic
Specifications}, whose semantics we will use in the present paper.

Briefly, ELPs consist of sets of \emph{rules} of the form
\[
  a_1 \vee \ldots \vee a_n \gets \ell_1, \ldots, \ell_m.
\]
In these rules, all $a_i$ are \emph{atoms} of the form $p(t_1, \ldots, t_n)$,
where $p$ is a predicate name, and $t_1, \ldots, t_n$ are terms, that is, either
variables or constants. Each $\ell$ is either an objective or subjective
literal, where objective literals are of the form $a$ or $\neg a$ ($a$ being an atom), and subjective literals are of the form $\eK l$ or $\neg \eK l$, where $l$
is an objective literal. Note that often the operator $\eM$ is also used, which
we will simply treat as a shorthand for $\neg \eK \neg$.

The domain of constants in an ELP $P$ is given implicitly by the set of all
constants that appear in it. Generally, before evaluating an ELP program,
variables are removed by a process called \emph{grounding}, that is, for every
rule, each variable is replaced by all possible combination of constants, and
appropriate ground copies of the rule are added to the resulting program
$\mathit{ground}(P)$. In practice, several optimizations have been implemented
in state-of-the-art systems that try to minimize the size of the grounding.

The result of a (ground) ELP program $P$ is calculated as follows~\citep{aaai:Gelfond91}. An \emph{interpretation} $I$ is a set of ground atoms
appearing in $P$. A set of interpretations $\calI$ satisfies a subjective
literal $\eK l$ (denoted $\calI \models \eK l$) iff the objective literal $l$ is
satisfied in all interpretations in $\calI$. The \emph{epistemic reduct}
$P^\calI$ of $P$ w.r.t.\ $\calI$ is obtained from $P$ by replacing all
subjective literals $\ell$ with either $\top$ in case where $\calI \models
\ell$, or with $\bot$ otherwise. $P^\calI$, therefore, is an ASP program, that
is, a program without subjective literals. The solutions to an ELP $P$ are
called \emph{world views}.  A set of interpretations $\calI$ is a world view of
$P$ iff $\calI = \answersets{P^\calI}$, where
$\answersets{P^\calI}$ denotes the set of stable models (or answer sets) of the
logic program $P^\calI$ according to the semantics of answer set programming~\citep{ngc:GelfondL91}. Checking whether a world view exists for an ELP is known
to be \SIGMA{P}{3}-complete in general, as shown by \cite{birthday:Truszczynski11}.

\section{Reversibility of Actions}\label{sec:reversibility}

In this section, we focus on the notion of uniform reversibility, which is a subclass of action reversibility as explained in detail by \cite{MorakCFF20}.
Intuitively, we call an action reversible if there is a way to undo all the
effects that this action caused, and we call an action \emph{uniformly
reversible} if its effects can be undone by a single sequence of actions
irrespective of the state where the action was applied.

While this intuition is fairly straightforward, when formally defining this
concept, we also need to take several other factors into account---in
particular, the set of possible states where an action is considered plays
an important role~\citep{MorakCFF20}.

\begin{definition}\label{def:uniformreversibility}
  Let $\calF$ be a set of facts, $\calA$ be a set of actions, $S \subseteq 2^\calF$ be 
  a set of states, and $a \in \calA$ be an action. We call $a$
  \emph{uniformly $S$-reversible} iff there exists a sequence of actions $\pi = \langle a_1,
  \ldots, a_n \rangle \in \calA^n$ such that for each $s \in S$ wherein $a$ is
  applicable it holds that $\pi$ is applicable in $\applyactions{a}{s}$ and
  $\applyactions{\pi}{\applyactions{a}{s}} = s$.
\end{definition}

The notion of uniform reversibility in the
most general sense does not depend on a concrete STRIPS planning task, but only
on a set of possible actions and states w.r.t.\ a set of facts. Note that the
set of states $S$ is an explicit part of the notion of uniform $S$-reversibility.

Based on this general notion, it is then possible to define several concrete
sets of states $S$ that are useful to consider when considering whether an
action is reversible.
For instance, $S$ could be defined via a propositional formula over the
facts in $\calF$.
Or we can consider a set of all possible states ($2^\calF$) which gives us
a notion of universal reversibility that applies to all possible planning
tasks that share the same set of facts and actions (i.e., the tasks that
differ only in the initial state or goals).
Or we can move our attention to a specific STRIPS instance and ask whether
a certain action is uniformly reversible for all states reachable from the
initial state.

\begin{definition}\label{def:uniformreversibility:names}
  Let $\calF$, $\calA$, $S$, and $a$ be as in
  Definition~\ref{def:uniformreversibility}. We call the action $a$
  \begin{enumerate}
    \item \emph{uniformly $\phi$-reversible} iff $a$ is uniformly $S$-reversible in the set $S$ of
      models of the propositional formula $\phi$ over $\calF$;
    \item \emph{uniformly reversible in $\Pi$} iff $a$ is uniformly $\calR_\Pi$-reversible for some
      STRIPS planning task $\Pi$; and
    \item \emph{universally uniformly reversible}, or, simply, \emph{uniformly reversible}, iff $a$
      is uniformly $2^\calF$-reversible.
  \end{enumerate}
\end{definition}


Given the above definitions, we can already observe some interrelationships. In
particular, universal uniform  reversibility (that is, uniform
reversibility in the set of all possible
states) is obviously the strongest notion, implying all the other, weaker
notions. It may be particularly important when one wants to establish
uniform reversibility
irrespective of the concrete STRIPS instance. On the other hand,
$\phi$-reversibility may be of particular interest when $\phi$ encodes the
natural domain constraints for a given planning task.

  %
  %


  %
  %
  %

The notion of uniform reversibility naturally gives rise to the notion of the
reverse plan. We say that some action $a$ has an \emph{($S$-)reverse plan} $\pi$
iff $a$ is uniformly ($S$-)reversible using the sequence of actions $\pi$. It is
interesting to note that this definition of the reverse plan based on uniform
reversibility now coincides with the same notion as defined by
\cite{japll:EiterEF08}. Note, however, that
in that paper the authors use a much more general planning language. 

Even if the length of the reverse plan is polynomially bounded, the problem of
deciding whether an action is uniformly ($\phi$-)reversible is intractable. In
particular, deciding whether an action is universally uniformly reversible
(resp. uniformly $\phi$-reversible) by a polynomial length reverse plan is
NP-complete (resp. in $\SIGMA{P}{2}$)~\citep{MorakCFF20}.

\section{Methods}\label{sec:methods}

After reviewing the relevant features of \emph{plasp}, described by \cite{DimopoulosGLRS19}, in
Section~\ref{sec:plasp}, we present our encodings for determining reversibility
in Section~\ref{sec:asprev}.

\subsection{The \emph{plasp} Format}\label{sec:plasp}

The system \emph{plasp}, described by \cite{DimopoulosGLRS19}, transforms PDDL domains and
problems into facts. Together with suitable programs, plans can then be computed
by ASP solvers---and hence also by ELP solvers, since ELPs are a superset of ASP
programs.
Given a STRIPS domain with facts $\calF$ and actions $\calA$, the following relevant facts and rules will be created by  \emph{plasp}:

\begin{itemize}
  \item \verb!variable(variable("f")).! for all \verb!f! $\in \calF$
  \item \verb!action(action("a")).! for all \verb!a! $\in \calA$
  \item \verb!precondition(action("a"),variable("f"),value(variable("f"),true))!\\
 \verb!:- action(action("a")).!\\ for each \verb!a! $\in \calA$ and \verb!f! $\in \precond{\mbox{a}}$
\item \verb!postcondition(action("a"),effect(unconditional),variable("f"),!\\
 \verb!              value(variable("f"),true)) :- action(action("a")).!\\ for each \verb!a! $\in \calA$ and \verb!f! $\in \addeffects{\mbox{a}}$
\item \verb!postcondition(action("a"),effect(unconditional),variable("f"),!\\
  \verb!              value(variable("f"),false)) :- action(action("a")).!\\ for each \verb!a! $\in \calA$ and \verb!f! $\in \deleffects{\mbox{a}}$
\end{itemize}

In addition, a predicate \verb!contains! encodes all possible values for a given
variable (for STRIPS, this is either true or false).

\begin{example}
  
The STRIPS domain with $\calF=\{f\}$ and actions $del$-$f=\langle\{f\},\emptyset,\{f\}\rangle$ and $add$-$f=\langle\emptyset,\{f\},\emptyset\rangle$ is written in PDDL as follows:

{\small
\begin{verbatim}
  (define (domain example1)
  (:requirements :strips)
  (:predicates (f) )
  (:action del-f
   :precondition (f)
   :effect (not (f)))
  (:action add-f
   :effect (f)))
\end{verbatim}
}

\emph{plasp} translates this domain to the following rules (plus a few technical facts and rules):

{\small
\begin{verbatim}
  variable(variable("f")).
  action(action("del-f")).
  precondition(action("del-f"), variable("f"), value(variable("f"), true)) :-
                action(action("del-f")).
  postcondition(action("del-f"), effect(unconditional), variable("f"),
                value(variable("f"), false)) :- action(action("del-f")).
  action(action("add-f")).
  postcondition(action("add-f"), effect(unconditional), variable("f"),
                value(variable("f"), true)) :- action(action("add-f")).
\end{verbatim}
}
\end{example}

\subsection{Reversibility Encodings using ASP and ELPs}\label{sec:asprev}

In this section, we present our ASP and ELP encodings for checking whether, in a
given domain, there is an action that is uniformly reversible. As we have seen
in Section~\ref{sec:plasp}, the \emph{plasp} tool is able to rewrite STRIPS
domains into ASP rules even when no concrete planning instance for that domain
is given.  We will present two encodings, one for (universal) uniform
reversibility, and one that can be used for uniform $\phi$-reversibiliy.

Note that \emph{universal} uniform reversibility is computationally easier than
$\phi$-uniform reversibility (under standard complexity-theoretic assumptions).
For a given action (and polynomial-length reverse plans), the former can be
decided in \NP, while the latter is harder (Theorem~18 and~20 in \cite{MorakCFF20}).
We will hence start with the encoding for the former problem, which follows a
standard guess-and-check pattern.

\subsubsection{Universal Uniform Reversibility}

The encodings are based on \verb!sequential-horizon.lp! in the \emph{plasp} distribution. 

\paragraph{ELP Encoding.} As a ``database'' the encoding
takes the output of \emph{plasp}'s translate action (for details, see~\citep{DimopoulosGLRS19}). The
problem can be solved in \NP\ due to the following Observation (*): in any
(universal) reverse plan for some action $a$, it is sufficient to consider only
the set of facts that appear in the precondition of $a$. If any action in a
candidate reverse plan $\pi$ for $a$ (resp.\ $a$ itself) contains any other fact
than those in $\precond{a}$, then $\pi$ cannot be a reverse plan for $a$ (resp.\
$a$ is not uniformly reversible), see Theorem~18 in \cite{MorakCFF20} or Theorem~3 in \cite{ChrpaFM21}. With this
observation in mind, we can now describe the (core parts of) our
encodings\footnote{The full encodings are available here:
\url{https://seafile.aau.at/d/cd4cb0d65d124a619920/}.}. We start with our ELP
encoding and will explain later how to modify it to obtain a plain ASP encoding. We should note that here the epistemic operators are used in a way as choices are used in ASP. We did this in order to understand the computational overhead of using ELP rather than ASP, but also in preparation for the uniform $\phi$-reversibility encoding.

The ELP encoding makes use of the following main predicates (in addition to
several auxiliary predicates, as well as those imported from \emph{plasp}):
\begin{itemize}
  \item \verb!chosen/1! encodes the action to be tested for reversibility.
  \item \verb!holds/3! encodes that some fact (or variable, as they are called
    in \emph{plasp} parlance) is set to a certain value at a given time step.
  \item \verb!occurs/2! encodes the candidate reverse plan, saying which action
    occurs at which time step.
\end{itemize}

With the intuitive meaning of the predicates defined, we first choose a single
action from the available actions and set the initial state as the facts in the
precondition of the chosen action. The first two lines partition the actions into chosen and unchosen ones; since it is a ``modal guess,'' there will be one world view for each partition. The third line makes sure that there is at most one chosen action, and lines 4 and 5 enforce at least one chosen action. The last rule says, in line with the Observation (*)
above, that only those variables in the precondition are relevant to check for a
reverse plan.
\small
\begin{verbatim}
chosen(A) :- action(action(A)), not &k{-chosen(A)}.
-chosen(A) :- action(action(A)), not &k{ chosen(A)}.
:- chosen(A), chosen(B), A!=B.
onechosen :- chosen(A).
:- not onechosen.

holds(V, Val, 0) :-
  chosen(A), precondition(action(A), variable(V), value(variable(V), Val)).
relevant(V) :- holds(V, _, 0).
\end{verbatim}
\normalsize

These rules set the stage for the inherent planning problem to be solved to find
a reverse plan. In fact, from the initial state guessed above, we need to find a
plan $\pi$ that starts with action $a$ (the chosen action), such that after
executing $\pi$ we end up in the initial state again. Such a plan is a
(universal) reverse plan. This idea is encoded in the following:
\small
\begin{verbatim}
time(0..horizon+1).

occurs(A, 1) :- chosen(A).
occurs(A, T) :- action(action(A)),time(T), T > 1, not &k{-occurs(A, T)}.
-occurs(A, T) :- action(action(A)),time(T), T > 1, not &k{occurs(A, T)}.

:- occurs(A,T), occurs(B,T), A!=B.
oneoccurs(T) :- occurs(A,T), time(T), T > 0.
:- time(T), T>0, not oneoccurs(T).

caused(V, Val, T) :-
  occurs(A, T), postcondition(action(A), _, variable(V), value(variable(V), Val)).
modified(V, T) :- caused(V, _, T).

holds(V, Val, T) :- caused(V, Val, T).
holds(V, Val, T) :- holds(V, Val, T - 1), not modified(V, T), time(T).
\end{verbatim}
\normalsize

The above rules guess a potential plan $\pi$ using the same technique as above, and then
execute the plan on the initial state (changing facts if this is caused by the
application of a rule, and keeping the same facts if they were not modified).
Finally, we simply need to check that the plan is (a) executable, and (b) leads
from the initial state back to the initial state. This can be done with the
following constraints:

\small
\begin{verbatim}
:- occurs(A, T), precondition(action(A), variable(V), value(variable(V), Val)),
   not holds(V, Val, T - 1).

:- occurs(A, T), precondition(action(A), variable(V), _), not relevant(V).
:- occurs(A, T), postcondition(action(A), _, variable(V), _), not relevant(V).

noreversal :- holds(V, Val, 0), not holds(V, Val, H+1), horizon(H).
noreversal :- holds(V, Val, H+1), not holds(V, Val, 0), horizon(H).
:- not &k{ ~ noreversal}.
\end{verbatim}
\normalsize

The first rule checks that rules in the candidate plan are actually applicable.
The next two check that the rules do not contain any facts other than those that
are relevant (cf.\ observation (*) above). Finally, the last three rules make
sure that at the maximum time point (i.e.\ the one given by the externally
defined constant ``horizon'') the initial state and the resulting state of plan
$\pi$ are the same. It is not difficult to verify that any world view of the
above ELP (combined with the \emph{plasp} translation of a STRIPS problem
domain) will yield a plan $\pi$ (encoded by the \verb!occurs! predicate) that
contains the sequence of actions $a, a_1, \ldots, a_n$, where $a_1, \ldots, a_n$
is a (universal) reverse plan for the action $a$ (each world view consists of
precisely one answer set). Note that our encoding yields reverse plans of length
exactly as long as set in the ``horizon'' constant. One could for instance employ an iterative deepening approach for determining the shortest reverse plans in case the plan length is not known or fixed. This completes our ELP
encoding for the problem of deciding universal uniform reversibility.

We can show that the encoding indeed leads to the correct result:

\begin{theorem}\label{thm:elpcorrect}
  Given a STRIPS planning task $\Pi = \langle \calF, \calA, s_0, G \rangle$, the
  ELP encoding in this section, when applied to $\Pi$, produces exactly one
  world view for each universally uniformly reversible action $a \in \calA$ and
  reverse plan $\pi$ of length \verb!horizon! for $a$.
\end{theorem}

\begin{proof}[Proof (Sketch).]
  We will show that, for each such action $a$ and reverse plan $\pi$, there
  exists exactly one world view $\calI$, such that every answer set in $\calI$
  contains the facts \verb!chosen(a)! and \verb!occurs(a', i)! for each action
  $a' \in \pi$, where $a'$ is the $(i-1)$-th action in $\pi$. This follows by
  construction:
  
  The rules deriving the \verb!chosen! and \verb!occurs! predicates, together
  with the constraints that follow, ensure that there is exactly one world view
  candidate per choosable action and candidate reverse plan of length
  \verb!horizon!. Because of Theorem~18 in \cite{MorakCFF20}, for universal
  uniform reversibility we only need to check a single starting state, and hence
  each world view candidate $\calI$ has at most one answer set $M$, i.e.\ $\calI
  = \{ M \}$.

  The rules deriving the \verb!holds! and \verb!caused! predicates then execute
  action $a$ and the reverse plan, keeping track of which value each variable
  has after each step (represented by time points \verb!T!).  Finally, $M$ is
  eliminated as an answer set in case where some action $a'$ in the reverse plan
  is not applicable or if $a'$ ``touches'' a variable that does not occur in the
  precondition of the chosen action $a$ (encoded in the predicate
  \verb!relevant!). The latter check is, again, correct because of Observation
  (*). The final three rules ensure that, in any world view, no answer set can
  contain the fact \verb!noreversal!, which is true if and only if some variable
  in the initial state (time point \verb!0!) has a different value from the
  final state (time point \verb!horizon + 1!).

  Hence, in any remaining world view $\calI = \{ M \}$, $M$ contains precisely a
  chosen action $a$, a reverse plan $\pi$ of length \verb!horizon! inside
  the \verb!occurs! predicate, and the intermediate states at each time
  step after the successful and valid application of action $a$ or actions
  from $\pi$, starting at some initial state that equals the final state. But
  this is precisely a reverse plan for $a$ of length \verb!horizon!, as desired.
\end{proof}

\paragraph{ASP Encoding.} Now, to see how the same thing can be achieved using
ASP, we can modify the encoding above as follows, yielding an encoding that
guarantees that every answer set represents a possible uniform reverse plan.
Firstly, in order to choose the action to reverse, the first five rules of the
ELP encoding can be replaced by a simple choice rule:

\begin{verbatim}
1 {chosen(A) : action(action(A))} 1.
\end{verbatim}

Similarly, the rules that chose, for each time step, an action (via the
\verb!occurs! predicate), can be replaced with a choice rule as follows:

\begin{verbatim}
1 {occurs(A, T) : action(action(A))} 1 :- time(T), T > 1.
\end{verbatim}

Finally, the check that no reversal exists (represented by the \verb!noreversal!
atom in the ELP encoding) can be encoded in ASP using simple constraints:

\begin{verbatim}
:- holds(V, Val, 0), not holds(V, Val, horizon+1).
:- holds(V, Val, horizon+1), not holds(V, Val, 0).
\end{verbatim}

This completes the ASP encoding, which now does not contain any subjective
literals. It can be seen that, whereas the ELP encoding generates one world view
per uniform reverse plan (by doing all the guesses via subjective literals), the
ASP encoding will generate one answer set per such plan:

\begin{theorem}\label{thm:aspcorrect}
  Given a STRIPS planning task $\Pi = \langle \calF, \calA, s_0, G \rangle$, the
  ASP encoding in this section, when applied to $\Pi$, produces exactly one
  answer set for each universally uniformly reversible action $a \in \calA$ and
  reverse plan $\pi$ of length \verb!horizon! for $a$.
\end{theorem}

\begin{proof}[Proof (Idea).]
  The proof proceeds in a similar fashion to the proof of
  Theorem~\ref{thm:elpcorrect}. In particular, now the actions are not guessed
  via a world view, but directly inside the answer set, via the appropriate
  choice rules. Hence, candidate answer sets contain all combinations of chosen
  actions and reverse plan candidates. Via the constraints, any answer set where
  the actions in the reverse plan don't follow the conditions of Observation
  (*), or where they do not lead back to the original state, are eliminated,
  leaving only answer sets that contain chosen actions together with valid
  reverse plans for them, as desired.
\end{proof}

\paragraph{Comparison.} The ASP encoding is a fairly straightforward
guess-and-check program, as the underlying problem of deciding universal uniform
action reversibility is only \NP-complete \citep{MorakCFF20}. In this case, it
could be argued that the choice rules employed there are a more natural encoding
than guessing via the modal operators of ELPs. However, in terms of contrasting
the expressiveness of the two languages, we feel that, still, it is interesting
to see how ``simple'' \NP-complete problems can be encoded using the modal
operators of ELPs, as this may lead to further improvements of the modelling
capabilities of the ELP language in the future. It also stands to reason that,
in the future, ELP solvers should aim to provide some syntactic sugar for these
modal operators for guess-and-check programs, similar to how choice rules are
provided by modern ASP solvers.

\subsubsection{Other Forms of Uniform Reversibility}

\paragraph{ELP Encoding.} Using a similar
guess-and-check idea as in the previous encodings, we can also check for uniform
reversibility for a specified set of states (that is, uniform
$S$-reversibility). Generally, the set $S$ of relevant states is encoded in some
compact form, and our encoding therefore, intentionally, does not assume
anything about this representation, but leaves the precise checking of the set
$S$ open for implementations of a concrete use case. The predicates used in this
more advanced encoding are similar to the ones used in the previous for the
universal case above, and hence we will not list them here again. However, in
order to encode the for-all-states check (i.e.\ the check that the candidate
reverse plan works in \emph{all} states inside the set $S$), we now need our
world views to contain multiple answer sets: one for each state in $S$. We again
start off with the ELP encoding. However, this time, we will see afterwards that
there is no easy modification to immediately obtain an ASP encoding, but the two
differ substantially\footnote{The full encodings
can be found at \url{https://seafile.aau.at/d/cd4cb0d65d124a619920/}.}.

The ELP encoding starts off much like the previous one:
\small
\begin{verbatim}
chosen(A) :- action(action(A)), not &k{-chosen(A)}.
-chosen(A) :- action(action(A)), not &k{ chosen(A)}.

:- chosen(A), chosen(B), A!=B.
onechosen :- chosen(A).
:- not onechosen.

holds(V, Val, 0) :- chosen(A), 
  precondition(action(A), variable(V), value(variable(V), Val)).
\end{verbatim}
\normalsize

Note that we no longer need to keep track of any set of ``relevant'' facts,
since we now need to consider all the facts that appear inside the actions and
in the set $S$ of states. However, we need to open up several answer sets, one
for each state. This is done by guessing a truth value for each fact at time
step 0. Recall that \verb!contains! is part of the \emph{plasp} output, encoding all possible values for a given
variable.
\small
\begin{verbatim}
holds(V,Val,0) | -holds(V,Val,0) :- 
  variable(variable(V)), contains(variable(V),value(variable(V),Val)).

oneholds(V,0) :- holds(V,Val,0).
:- variable(variable(V)), not oneholds(V,0).
:- holds(V,Val,0), holds(V,Val1,0), Val != Val1.
\end{verbatim}
\normalsize

Next, we again guess and execute a plan, keeping track of whether the actions
were able to be applied at each particular time step:
\small
\begin{verbatim}
occurs(A, 1) :- chosen(A).
occurs(A, T) :- action(action(A)),time(T), T > 1, not  &k{-occurs(A, T)}.
-occurs(A, T) :- action(action(A)),time(T), T > 1, not  &k{occurs(A, T)}.

:- occurs(A,T), occurs(B,T), A!=B.
oneoccurs(T) :- occurs(A,T), time(T), T > 0.
:- time(T), T>0, not oneoccurs(T).

inapplicable :-  occurs(A, T),
  precondition(action(A), variable(V), value(variable(V), Val)),
  not holds(V, Val, T - 1).
:- not &k{ ~ inapplicable}.

caused(V, Val, T) :- occurs(A, T),
  postcondition(action(A), E, variable(V), value(variable(V), Val)).
modified(V, T) :- caused(V, _, T).

holds(V, Val, T) :- caused(V, Val, T).
holds(V, Val, T) :- holds(V, Val, T - 1), not modified(V, T), time(T).
\end{verbatim}
\normalsize

Again, the rules above choose a candidate reverse plan $\pi$, starting with the
action-to-be-checked $a$, as before. Furthermore, we check applicability: $\pi$
should be applicable (i.e.\ at each time step, the relevant action must have
been applied, encoded by the third block of rules above), and furthermore, only
modified facts (i.e.\ those affected by an action) can change their truth values
from time step to time step. Finally, we again need to make sure that the
guessed plan actually returns us to the original state at time step 0.
\small
\begin{verbatim}
noreversal :- holds(V, Val, 0), not holds(V, Val, H+1), horizon(H).
noreversal :- holds(V, Val, H+1), not holds(V, Val, 0), horizon(H).
:- not &k{ ~ noreversal}.
\end{verbatim}
\normalsize

This concludes the main part of our ELP encoding. In its current form, the encoding
given above produces exactly the same results as the first encoding given in
this section; that is, it checks for \emph{universal} uniform reversibility.
However, the second encoding can be easily modified in order to check uniform
$S$-reversibility. Simply add a rule of the following form to it:

\small
\begin{verbatim}
:- < check guessed state against set S >
\end{verbatim}
\normalsize

This rule should fire precisely when the current
guess (that is, the currently considered starting state) does not belong to the
set $S$. This can of course be generalized easily. For example, if set $S$ is
given as a formula $\phi$, then the rule should check whether the current guess
conforms to formula $\phi$ (i.e., encodes a model of $\phi$). Other compact
representations of $S$ can be similarly checked at this point. Hence, we have a
flexible encoding for uniform $S$-reversibility that is easy to extend with
various forms of representations of set $S$.

\paragraph{ASP Encoding.} Now, for the ASP encoding. As we will see, this is
now substantially more involved than the ELP encoding, since we need to apply
an encoding technique called \emph{saturation}, cf.~\cite{amai:EiterG95},
allowing us to express a form of universal quantification. We can start off the
same as last time, that is, with a choice rule:
\small
\begin{verbatim}
1 {chosen(A) : action(action(A))} 1.
holds(V, Val, 0) :- chosen(A),
  precondition(action(A), variable(V), value(variable(V), Val)).
affected(A, V) :- postcondition(action(A), _, variable(V), _).
\end{verbatim}
\normalsize

We note the first difference compared to the ELP encoding: we need to keep track
of all STRIPS facts that are potentially affected by an action. We assume that a
predicate \verb!opposites/2! exists that holds, in both possible orders, the
values ``true'' and ``false''. This will later be used to find the opposite
value of some STRIPS fact at a particular time step.

Next, we again guess and execute a plan:
\small
\begin{verbatim}
occurs(A, 1) :- chosen(A).
1 {occurs(A, T) : action(action(A))} 1 :- time(T), T > 1.

applied(0). % no action needs to be applied at time step 0
applicable(A, T) :- occurs(A, T), applied(T - 1),
  holds(V, Val, T - 1) :
    precondition(action(A), variable(V), value(variable(V), Val)).
applied(T) :- applicable(_, T).
holds(V, Val, T) :-  applicable(A, T),
  postcondition(action(A), _, variable(V), value(variable(V), Val)).
holds(V, Val, T) :- holds(V, Val, T - 1), occurs(A, T), applied(T),
  not affected(A, V).
\end{verbatim}
\normalsize

Note that we use the predicate \verb!affected! here to encode inertia for those
facts that are not affected by the applied action. From here on, we now see a
major difference to the ELP encoding. We need to set up our goal conditions and
then encode the universal check for all states of set $S$. First we check that
$\pi$ should be applicable (i.e.\ at each time step, the relevant action must
have been applied), and furthermore, the state at the beginning must be equal to
the state at the end.
\small
\begin{verbatim}
same(V) :- holds(V, Val, 0), holds(V, Val, horizon + 1).
samestate :- same(V) : variable(variable(V)).
planvalid :- applied(horizon + 1).
reversePlan :- samestate, planvalid.
\end{verbatim}
\normalsize

Finally, we need to specify that for all the states specified in the set $S$ the
candidate reverse plan must work. This is done as follows:
\small
\begin{verbatim}
holds(V, Val1, 0) | holds(V, Val2, 0) :-  variable(variable(V)),
  opposites(Val1, Val2), Val1 < Val2.
holds(V, Val, T) :- reversePlan, contains(variable(V), value(variable(V), Val)),
  time(T).
:- not reversePlan.
\end{verbatim}
\normalsize

As stated above, this is done using the technique of \emph{saturation}
\citep{amai:EiterG95}, allowing us to express a form of universal quantifier
that, in our case, checks that, for every state in the set $S$, we return to
the original state after applying the chosen action and the reverse plan. We
encourage the reader to refer to the relevant publication for more details on
the ``inner workings'' of this encoding technique. Again, as is, the ASP
encoding above checks \emph{universal} uniform reversibility. However, it again
can be easily modified in order to check uniform $S$-reversibility. Simply add
a rule of the following form to it, analogously to what we had for the ELP
encoding:

\small
\begin{verbatim}
reversePlan :- < check guessed state against set S >
\end{verbatim}
\normalsize

This completes the overview of our ELP and ASP encodings for uniform
reversibility.

\paragraph{Comparison.} Looking at the structure of the ELP and ASP encodings,
it is not difficult to see that they share a certain common structure. This is
not surprising, since they underlying language is the same. However, it can also
be observed that the technique of saturation, which is required (in terms of
expressive power) to encode uniform $S$-reversibility in ASP, is somewhat
non-intuitive, as it is not immediately clear, what the semantics of this
construction are. By contrast, the modal operators provided by ELPs make this
much more readable and declarative.

\subsection{Experiments}

We have conducted preliminary experiments with artificially constructed domains. The domains are as follows:

{\small
\begin{verbatim}
  (define (domain rev-i)
  (:requirements :strips)
  (:predicates (f0) ... (fi))
  (:action del-all
   :precondition (and  (f0) ... (fi) )
   :effect (and  (not (f0)) ... (not (fi))))
  (:action add-f0
   :effect (f0))
  ...
  (:action add-fi
   :precondition (fi-1)
   :effect (fi)))
\end{verbatim}
}
The action \verb!del-all! has a universal uniform reverse plan $\langle$
\verb!add-f0!, \ldots, \verb!add-fi! $\rangle$. We have generated instances from
\verb!i! $=1$ to \verb!i! $=6$ and from \verb!i! $=10$ to \verb!i! $=200$ with
step 10.  We have analyzed runtime and memory consumption of two problems: (a)
finding the unique reverse plan of size \verb!i! (by setting the constant
\verb!horizon! to \verb!i!) and proving that no other reverse plan exists, and
(b) showing that no reverse plan of length \verb!i-1! exists (by setting the
constant \verb!horizon!  to \verb!i-1!). We compare the four encodings described
in Section~\ref{sec:asprev}, and refer to the first two as the \emph{simple
ELP/ASP encoding} and to the second two as the \emph{general ELP/ASP encoding}.

We used plasp 3.1.1 (\url{https://potassco.org/labs/plasp/}), eclingo 0.2.0
(\url{https://github.com/potassco/eclingo}), and clingo 5.4.0
(\url{https://potassco.org/clingo/}) on a computer with a 2.3 GHz AMD EPYC 7601
CPU with 32 cores and 500 GB RAM running CentOS 8. We have set a timeout of 20
minutes and a memory limit of 16GB (which was never exceeded).

\begin{figure}
  \begin{tikzpicture}[scale=0.60]
\begin{axis}[
  xlabel=Number of facts,
  ylabel=Runtime (s)]
\addplot +[unbounded coords=jump] table [y=time(asp.simple),
x=size]{encodings/experiment/experiments.plan.dat};
\addlegendentry{simple ASP encoding}
\addplot +[unbounded coords=jump] table [y=time(asp.general),
x=size]{encodings/experiment/experiments.plan.dat};
\addlegendentry{general ASP encoding}
\addplot +[unbounded coords=jump] table [y=time(elp.simple),
x=size]{encodings/experiment/experiments.plan.dat};
\addlegendentry{simple ELP encoding}
\addplot +[unbounded coords=jump] table [y=time(elp.general),
x=size]{encodings/experiment/experiments.plan.dat};
\addlegendentry{general ELP encoding}
\end{axis}
\end{tikzpicture}
\qquad
  \begin{tikzpicture}[scale=0.60]
\begin{axis}[
  xlabel=Number of facts,
  ylabel=Memory (MB)]
\addplot +[unbounded coords=jump] table [y=memory(asp.simple),
x=size]{encodings/experiment/experiments.plan.dat};
\addlegendentry{simple ASP encoding}
\addplot +[unbounded coords=jump] table [y=memory(asp.general),
x=size]{encodings/experiment/experiments.plan.dat};
\addlegendentry{general ASP encoding}
\addplot +[unbounded coords=jump] table [y=memory(elp.simple),
x=size]{encodings/experiment/experiments.plan.dat};
\addlegendentry{simple ELP encoding}
\addplot +[unbounded coords=jump] table [y=memory(elp.general),
x=size]{encodings/experiment/experiments.plan.dat};
\addlegendentry{general ELP encoding}
\end{axis}
\end{tikzpicture}

    \caption{Calculating the unique reverse plan (plan length equals number of facts)}
\label{fig:experiment.exists}
\end{figure}
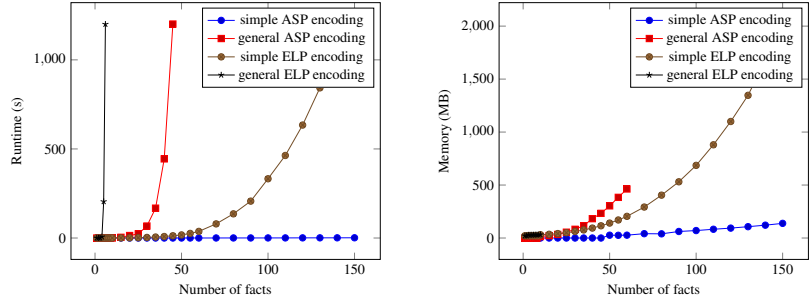

The results for problem (a) are plotted in Figure~\ref{fig:experiment.exists}.
The general ELP encoding exceeded the time limit already at the problem with
seven facts, while the simple ELP encoding could solve all problems with up to
150 facts within the time limit. The general and simple ASP encodings perform
better than their ELP counterparts, but the simple ELP encoding performed much
better than the saturation-based general ASP encoding, even though ELP solvers
are in their infancy compared to the heavily optimized ASP solving systems. The
memory consumption increased with \verb!i!  for all encodings, proportional to
the computation time.

\begin{figure}
  \begin{tikzpicture}[scale=0.60]
\begin{axis}[
  xlabel=Number of facts,
  ylabel=Runtime (s)]
\addplot +[unbounded coords=jump] table [y=time(asp.simple),
x=size]{encodings/experiment/experiments.noplan.dat};
\addlegendentry{simple ASP encoding}
\addplot +[unbounded coords=jump] table [y=time(asp.general),
x=size]{encodings/experiment/experiments.noplan.dat};
\addlegendentry{general ASP encoding}
\addplot +[unbounded coords=jump] table [y=time(elp.simple),
x=size]{encodings/experiment/experiments.noplan.dat};
\addlegendentry{simple ELP encoding}
\addplot +[unbounded coords=jump] table [y=time(elp.general),
x=size]{encodings/experiment/experiments.noplan.dat};
\addlegendentry{general ELP encoding}
\end{axis}
\end{tikzpicture}
\qquad
  \begin{tikzpicture}[scale=0.60]
\begin{axis}[
  xlabel=Number of facts,
  ylabel=Memory (MB)]
\addplot +[unbounded coords=jump] table [y=memory(asp.simple),
x=size]{encodings/experiment/experiments.noplan.dat};
\addlegendentry{simple ASP encoding}
\addplot +[unbounded coords=jump] table [y=memory(asp.general),
x=size]{encodings/experiment/experiments.noplan.dat};
\addlegendentry{general ASP encoding}
\addplot +[unbounded coords=jump] table [y=memory(elp.simple),
x=size]{encodings/experiment/experiments.noplan.dat};
\addlegendentry{simple ELP encoding}
\addplot +[unbounded coords=jump] table [y=memory(elp.general),
x=size]{encodings/experiment/experiments.noplan.dat};
\addlegendentry{general ELP encoding}
\end{axis}
\end{tikzpicture}

    \caption{Determining nonexistence of a reverse plan (plan length one step too short)}
\label{fig:experiment.notexists}

\end{figure}
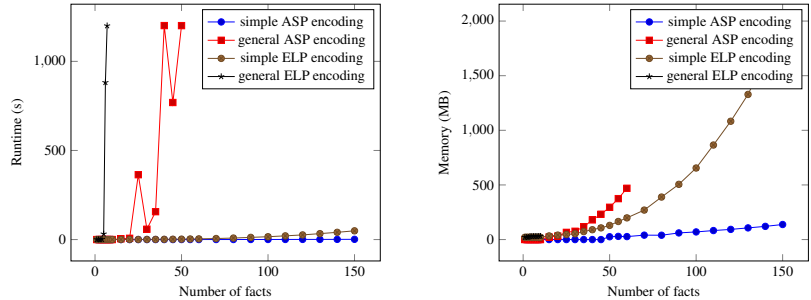

The results for problem (b) are plotted in
Figure~\ref{fig:experiment.notexists}. Interestingly, compared to (a), all the
encodings performed significantly better. While the general encodings still hit
the time limit for seven facts (ELP) and 50 facts (ASP), the simple encodings
were able to solve all the instances up to our maximum of \verb!i! $=250$ (the
figure stops at \verb!i! $=150$), but at the expense of increasing memory usage.

In total, the general encodings, for both ASP and ELP, scale worse, as
expected, since the ELP solver needs to evaluate all answer sets inside each
possible world view, and the ASP solver needs to compute the result of the
saturation check. However, for the simple encodings, especially the task of
testing for non-reversibility performed surprisingly well for both ASP and ELP.
From all of our results, however, we can see that ELP solving still severely
trails, in terms of performance, encodings for plain ASP. This was somewhat
expected, since ELP solvers are nowhere near as optimized as modern ASP
systems. We hope, however, that our results encourage further improvements in
the area of ELP solvers, since matching the ASP results, at least in this
particular benchmark set, does not seem completely out of reach.

\section{Conclusions}\label{sec:conclusions}

In this paper, we have given a review of several notions of action reversibility
in STRIPS planning, as originally presented by \cite{MorakCFF20}. We then
proceeded, on the basis of the PDDL-to-ASP translation tool \emph{plasp}, described by
\cite{DimopoulosGLRS19}, to present two ELP encodings and two ASP encodings to
solve the task of universal uniform reversibility of STRIPS actions, given a
corresponding planning domain. When given to an appropriate solving system,
these encodings, combined with the ASP translation of STRIPS planning domains
produced by \emph{plasp}, then yield a set of world views (for ELP) or answer
sets (for ASP), each one representing a (universal) reverse plan for each action
in the domain, for which such a reverse plan could be found.

The four encodings use two different approaches. The first, simpler, encoding
makes use of a shortcut that allows it to focus only on those facts that appear
in the precondition of the action to check for reversibility, as described
by~\cite{MorakCFF20}.
The second two encodings make use of the power of world views containing
multiple answer sets in ELP, and the encoding technique of saturation
as of \cite{amai:EiterG95} in ASP, respectively, which allows for encoding universal
quantifiers. These two encodings try to directly represent the original
definition of uniform reversibility: for an action to be uniformly reversible,
there must exist a plan, and this plan must revert the action in all possible
starting states (where it is applicable). Hence, the two general encodings are
more flexible insofar as they also allows for the checking of non-universal
uniform reversibility (e.g.\ to check for uniform $\phi$-reversibility, where
the starting states are given via some formula $\phi$).

In order to compare the four encodings, we performed some benchmarks on
artificially generated instances by checking whether there is an action that is
universally uniformly reversible. For the ELP and ASP communities, it will not
come as a surprise that the ELP encodings perform worse than the ASP encodings.
We see this as a call-to-action to further optimize and improve ELP solvers.
From our experiments, it seems that the performance of ASP solvers, while
significantly better, is not out of reach for ELP systems.

For future work, we intend to optimize our ELP encodings further, and test them
with other established ELP solvers. There are several competing ELP semantics
out there and several solvers are available. It would also be interesting to see
how the encodings perform when compared to a procedural implementation of the
algorithms proposed for reversibility checking by \cite{MorakCFF20}. We would also like to compare our approach to existing tools
\emph{RevPlan}\footnote{\url{http://www.kr.tuwien.ac.at/research/systems/revplan/index.html}}
(implementing techniques of \cite{japll:EiterEF08}) and \emph{undoability}
(implementing techniques of \cite{icaps:DaumT0HW16}). Furthermore, we aim to
explore how our techniques can be extended to planning languages more expressive
than STRIPS. We envision various avenues for that, one is to deal with ``lifted representations'' (going beyond propositional atoms), another one is to allow for non-deterministic action effects or exogenous events, for which ASP and ELP seem to be well-suited.

\section*{Acknowledgements}
	  Supported by the S\&T Cooperation CZ 05/2019 ``Identifying Undoable
	  Actions and Events in Automated Planning by Means of Answer Set
	  Programming'', by the Czech Ministry of Education, Youth and Sports
	  under the Czech-Austrian Mobility programme (project no. 8J19AT025),
	  by the OP VVV funded project CZ.02.1.01/0.0/0.0/16\_019/0000765
	  ``Research Center for Informatics'' and by the Czech Science
	  Foundation (project no. 18-07252S).

\bibliographystyle{acmtrans}
\bibliography{references}

\newpage
\appendix
\section{Example}

As an example, consider the following domain, which follows the pattern of the experiments:

{\small
\begin{verbatim}
(define (domain rev-2)
  (:requirements :strips)
  (:predicates (f0) (f1) )

  (:action del-all
   :precondition (and  (f0) (f1) )
   :effect (and  (not (f0)) (not (f1)) ) )

  (:action add-f0
   :effect (f0) )

  (:action add-f1
   :precondition (f0)
   :effect (f1) )
)
\end{verbatim}
}

The tool \emph{plasp} translates it to the following ASP quasi-facts:

\scriptsize
\begin{verbatim}
boolean(true).
boolean(false).

% types
type(type("object")).

% variables
variable(variable("f0")).
variable(variable("f1")).

contains(X, value(X, B)) :- variable(X), boolean(B).

% actions
action(action("del-all")).
precondition(action("del-all"), variable("f0"), value(variable("f0"), true))
  :- action(action("del-all")).
precondition(action("del-all"), variable("f1"), value(variable("f1"), true))
  :- action(action("del-all")).
postcondition(action("del-all"), effect(unconditional), variable("f0"), value(variable("f0"), false))
  :- action(action("del-all")).
postcondition(action("del-all"), effect(unconditional), variable("f1"), value(variable("f1"), false))
  :- action(action("del-all")).

action(action("add-f0")).
postcondition(action("add-f0"), effect(unconditional), variable("f0"), value(variable("f0"), true))
  :- action(action("add-f0")).

action(action("add-f1")).
precondition(action("add-f1"), variable("f0"), value(variable("f0"), true))
  :- action(action("add-f1")).
postcondition(action("add-f1"), effect(unconditional), variable("f1"), value(variable("f1"), true))
  :- action(action("add-f1")).
\end{verbatim}
\normalsize

In the simple ELP encoding, the first rules will give rise to multiple possible world views, one that contains answer sets with \verb!chosen("del-all"))!, \verb!occurs("del-all",1)!, \verb!holds("f0",true,0)!, \verb!holds("f1",true,0)!, \verb!relevant("f0")!, and \verb!relevant("f1")!, one world view that contains answer sets with \verb!chosen("add-f0")! and \verb!occurs("add-f0",1)!, and one with answer sets containing \verb!chosen("add-f1")!, \verb!occurs("add-f1",1)!, \verb!holds("f0",true,0)!, and \verb!relevant("f0")!.

When we set the constant \verb!horizon! to 2, more world views are created, based on the three mentioned above, one for each pair of actions of the three available ones. Each world view will contain at most one answer set. Many of these answer sets turn out to be invalid immediately, for instance any answer set containing \verb!occurs("add-f0",1)! and \verb!occurs("add-f1",2)! will be violating a constraint, as the precondition of \verb!add-f1! is not relevant. Others are invalidated because the preconditions of actions are not met. A few others derive \verb!noreversal!, for instance \verb!occurs("del-all",1)!, \verb!occurs("add-f0",2)!, \verb!occurs("add-f0",3)!, as we have \verb!holds("f1",true,0)! but not \verb!holds("f1",true,3)!.

We can check that the only world view with an answer set, in which \verb!noreversal! is not derived, is the one in which \verb!occurs("del-all",1)!, \verb!occurs("add-f0",2)!, \verb!occurs("add-f1",3)! hold. Indeed, \verb!del-all! is the only universally uniformly reversible action, and its only reverse plan of length 2 is $\langle$ \verb!add-f0!, \verb!add-f1! $\rangle$.

The simple ASP encoding works in a very similar way. Since the simple ELP encoding has at most one answer set per world view, we simple turn the ``epistemic guesses'' into ``standard guesses'', so instead of an answer set encapsulated in a world view, it is just an answer set, and also there, one answer set exists for the example, in which \verb!occurs("del-all",1)!, \verb!occurs("add-f0",2)!, \verb!occurs("add-f1",3)! hold.

Concerning the general ELP encoding, similar world views as above are created. But in that encoding, multiple answer sets can exist in a world view: for each variable not in the precondition of the chosen action, there will be answer sets in which the variable is true, and answer sets in which it is false. So, any world view, in which \verb!chosen("del-all"))!, \verb!occurs("del-all",1)!, \verb!holds("f0",true,0)!, \verb!holds("f1",true,0)! hold, will still have at most a single answer set, as all variables occur in the precondition of \verb!del-all!. It is easy to see that the reverse plan is then in a single-answer-set world view similar to the one in the simple ELP encoding.

On the other hand, in world views containing \verb!chosen("add-f0")! and \verb!occurs("add-f0",1)! four potential answer sets can exist, one with \verb!holds("f0",true,0)! and \verb!holds("f1",true,0)!, one with \verb!holds("f0",false,0)! and \verb!holds("f1",true,0)!, one with \verb!holds("f0",true,0)! and \verb!holds("f1",false,0)!, and one with \verb!holds("f0",false,0)!, \verb!holds("f1",false,0)!.

Let us have a look at the world view containing \verb!occurs("add-f0",1)!, \verb!occurs("del-all",2)!, \verb!occurs("del-all",3)!. For this, \verb!inapplicable! will be derived because the preconditions for \verb!occurs("del-all",3)! are not met in any of the answer sets, and also because the precondition for \verb!occurs("del-all",2)! is not met in those answer sets in which \verb!holds("f1",false,0)! is true. Therefore, the constraint \verb!:- not &k{ ~ inapplicable}.! is violated for this world view.

Maybe more interesting is the world view containing \verb!occurs("add-f0",1)!, \verb!occurs("add-f1",2)!, \verb!occurs("del-all",3)!. Here, \verb!inapplicable! will not be derived, as the preconditions of the actions hold in all answer sets. But consider the answer set containing \verb!holds("f0",true,0)! and \verb!holds("f1",true,0)!: neither \verb!holds("f0",true,3)! nor \verb!holds("f1",true,3)! is derived, so \verb!noreversal! is derived. \verb!noreversal! is also true in the other answer sets except the one containing \verb!holds("f0",false,0)! and \verb!holds("f1",false,0)!. The constraint \verb!:- not &k{ ~ noreversal}.! is thus violated for this world view.

The general ASP encoding works in a rather different way. Here, one candidate answer set will be created for each action to be reversed, one completion of the initial state and one candidate reverse plan.

So there is still only one answer set candidate containing \verb!chosen("del-all"))!, \verb!occurs("del-all",1)!, \verb!holds("f0",true,0)!, \verb!holds("f1",true,0)!. For \verb!occurs("add-f0",1)!, \verb!occurs("del-all",2)!, \verb!occurs("del-all",3)! there will be four answer set candidates, similar to the four answer sets in the world view in the general ELP encoding, similar for \verb!occurs("add-f0",1)!, \verb!occurs("del-all",2)!, \verb!occurs("del-all",3)!, there will also be four answer set candidates.

Let us first see what happens with the reverse plan answer set candidate containing \verb!chosen("del-all"))!, \verb!occurs("del-all",1)!, \verb!holds("f0",true,0)!, \verb!holds("f1",true,0)!. Eventually, \verb!samestate! and \verb!planvalid!, and \verb!reversePlan! are derived for this answer set, which means that \verb!holds(V,Val,T)! will be derived for all combinations of variables, values and times. It should be noted that doing this will not invalidate any derivation done earlier, so it is an answer set.

Now consider the answer set candidates containing \verb!occurs("add-f0",1)!, \verb!occurs("del-all",2)!, \verb!occurs("del-all",3)!. In none of these, \verb!planvalid! will be derived, as either the preconditions of \verb!occurs("del-all",2)! are not met (in answer set candidates having \verb!holds("f1",false,0)!) or the preconditions of \verb!occurs("del-all",3)! are not met. So, \verb!reversePlan! will not hold in any of these answer set candidates, which means that they all violate the constraint \verb!:- not reversePlan.!

Now for the answer set candidates containing \verb!occurs("add-f0",1)!, \verb!occurs("add-f1",2)!, \verb!occurs("del-all",3)!, \verb!planvalid! will be derived in all of them, but \verb!samestate! only in the one containing \verb!holds("f0",false,0)! and \verb!holds("f1",false,0)!. This means, that in the other three answer set candidates \verb!reversePlan! will not hold, violating  the constraint \verb!:- not reversePlan.! For the remaining one, the constraint is satisfied, however the saturation rule derives \verb!holds(V,Val,T)! for all combinations of variables, values and times. This ``inflated'' answer set candidate is not stable any longer: we can form a subset of it, in which exactly the atoms of one of the other three constraint-violating candidates are true plus \verb!reversePlan!; the obtained interpretation also satisfies the program and therefore is a counterexample for the stability of the saturated interpretation. One might object that \verb!reversePlan! is unsupported in the counterexample, but supportedness is not a requirement for countermodels. So none of the candidate answer sets containing \verb!occurs("add-f0",1)!, \verb!occurs("add-f1",2)!, \verb!occurs("del-all",3)! turned out to be  answer sets (for quite different reasons).

\label{lastpage}
\end{document}